\documentclass{llncs}

\usepackage{graphicx} 

\usepackage{algorithm}
\usepackage{algorithmic}
\usepackage{dsfont}

\usepackage{amsmath,amsfonts,amssymb}

\newcommand{\eqdef}{:=}
\newcommand{\what}{\widehat}
\renewcommand{\P}{\mathbb{P}}
\newcommand{\E}{\mathbb{E}}
\newcommand{\Esp}{\mathbb{E}}
\newcommand{\cA}{\mathcal{A}}
\newcommand{\cH}{\mathcal{H}}
\newcommand{\cS}{\mathcal{S}}

\newcommand{\cO}{\mathcal{O}}
\newcommand{\cR}{\mathcal{R}}

\newcommand{\wt}{\widetilde}

\newcommand{\veps}{\epsilon}
\newcommand{\geps}{\tilde{\veps}}

\newcommand{\s}{\dot{s}}

\newcommand{\agg}{{\rm agg}}

\newcommand{\Real}{\mathbb{R}}
\newcommand{\Nat}{\mathbb{N}}
\newcommand{\Hist}{\mathcal{H}}

\newcommand{\argmax}{\mathop{\mathrm{argmax}}}

\renewcommand{\v}[1]{\mathbf{#1}}
\newcommand{\pen}{\mathbf{pen}}
\newcommand{\lob}{\mathbf{lob}}

\renewcommand{\span}{\lambda}

\newcommand{\cM}{\mathcal{M}}

\newcommand{\AlgoName}{\texttt{OAMS}}

\newcommand{\pio}{{\pi^*}}
\newcommand{\phit}{{\phi^\circ}}

\newcommand{\st}{\textstyle}

\begin{document}

\pagestyle{empty}
\title{Selecting Near-Optimal Approximate State Representations in Reinforcement Learning}
\titlerunning{Selecting Near-Optimal Approximate State Representations in Reinforcement Learning}
\author{Ronald Ortner\inst{1} \and Odalric-Ambrym Maillard\inst{2} \and Daniil Ryabko\inst{3}}
\institute{Montanuniversitaet Leoben, Austria \and The Technion, Israel  \and Inria Lille-Nord Europe, \'{e}quipe SequeL, France, and Inria Chile \\
\email{rortner@unileoben.ac.at, odalric-ambrym.maillard@ens-cachan.org, daniil@ryabko.net}
}
\maketitle

\begin{abstract}
We consider a reinforcement learning setting introduced in~\cite{maimury}
where the learner does not have explicit access to the states of the underlying Markov decision process (MDP).
Instead, she has access to several models that map histories of past interactions to states.
Here we improve over known regret bounds in this setting, and more importantly generalize to the case where the 
models given to the learner do not contain a true model resulting in an MDP representation but only approximations of it.
We also give improved error bounds for state aggregation.
\end{abstract}

\section{Introduction}\label{sec:intro}
Inspired by \cite{frl}, in \cite{maimury} a reinforcement learning setting has been introduced where the learner does not have explicit information about the state space of the underlying Markov decision process (MDP). 
Instead, the learner has a set of \textit{models} at her disposal that map histories (i.e., observations, chosen actions and collected rewards) to states. However, only some models give a correct MDP representation. The first regret bounds in this setting were derived in \cite{maimury}. They recently have been improved in~\cite{oms} and extended to infinite model sets in \cite{iblb}. Here we extend and improve the results of~\cite{oms} as follows. First, we do not assume anymore that the model set given to the learner contains a true model resulting in an MDP representation. Instead, models will only approximate an MDP. Second, we improve the bounds of~\cite{oms} with respect to the dependence on the state space. 

For discussion of potential applications and related work on learning state representations in POMDPs (like predictive state representations \cite{psr}), we refer to \cite{maimury,oms,iblb}. Here we only would like to mention the recent work~\cite{hakama} that considers a similar setting, however is mainly interested in the question whether the true model will be identified in the long run, a question we think is subordinate to that of minimizing the regret, which means fast learning of optimal behavior.

\subsection{Setting}\label{sec:setting}
For each time step $t=1,2,\ldots$, let $\Hist_{t} := \cO\times(\cA\times\cR\times\cO)^{t-1}$ be the set of histories up to time $t$, where $\cO$ is the set of observations, $\cA$ a finite set of actions, and $\cR = [0,1]$ the set of possible rewards.
We consider the following reinforcement learning problem:
The learner receives some initial observation $h_{1} = o_1 \in \Hist_{1}= \cO$.
Then at any time step $t>0$, the learner chooses an action $a_t\in\cA$ based on the current history $h_{t} \in\Hist_{t}$,  
and receives an immediate reward~$r_{t}$ and the next observation~$o_{t+1}$ from the unknown environment.
Thus, $h_{t+1}$ is the concatenation of~$h_{t}$ with $(a_t,r_t,o_{t+1})$.
\smallskip

\textbf{State representation models.}
A \textit{state-representation model} $\phi$ is a function from the set of histories
$\Hist = \bigcup_{t\geq1}\Hist_t$ to a finite set of states $\cS_{\phi}$. 
A particular role will be played by state-representation models that induce a \textit{Markov decision process (MDP)}.
An MDP is defined as a decision process in which at any discrete time $t$, given action $a_t$, 
the probability of immediate reward $r_t$ and next observation $o_{t+1}$, given the past
history $h_t$, only depends on the current observation~$o_t$ 
i.e., $P(o_{t+1},r_{t}|h_ta_t) = P(o_{t+1},r_{t}|o_t,a_t)$, and this probability is also independent of~$t$. 
Observations in this process are called \textit{states} of the environment. We say that a state-representation model~$\phi$ is a \textit{Markov model} of the environment, if the process $(\phi(h_t),a_t,r_t), t\in\Nat$ is an MDP.
Note that such an MDP representation needs not be unique. In particular, we assume 
that we obtain a Markov model when mapping each possible history to a unique state.
Since these states are not visited more than once, this model is not very useful from the practical point of view, however.
In general, an MDP is denoted as $M(\phi)=(\cS_\phi,\cA,r,p)$, where $r(s,a)$ is the mean reward 
and $p(s'|s,a)$ the probability of a transition to state $s'\in\cS_\phi$ when choosing action $a\in \cA$ in state $s\in\cS_\phi$.

We assume that there is an \textit{underlying true} Markov model $\phi^\circ$ that gives a finite and \textit{weakly communicating} MDP, that is, for each pair of states $s,s'\in \cS^\circ:=\cS_{\phi^\circ}$ 
there is a $k\in\Nat$ and a sequence of actions $a_1,\ldots,a_k\in\cA$ such that the probability of reaching state $s'$
when starting in state $s$ and taking actions $a_1,\ldots,a_k$ is positive.
In such a weakly communicating MDP we can define the \textit{diameter} $D:=D(\phi^\circ):=D(M(\phi^\circ))$ to be the expected minimum time it takes to reach any state starting from any other state in the MDP~$M(\phi^\circ)$, cf.~\cite{jaorau}. In finite state MDPs, the \textit{Poisson equation} relates the average reward $\rho_\pi$ of any policy $\pi$ to the single step mean rewards and the transition probabilities. That is, for each \textit{policy}
$\pi$ that maps states to actions, it holds that
\begin{equation} \label{eq:poisson}
	\st \rho_\pi + \lambda_\pi(s) = r(s,\pi(s)) + \sum_{s'\in \cS^\circ}  p(s'|s,\pi(s))\cdot\lambda_\pi(s'),
\end{equation}
where $\lambda_\pi$ is the so-called \textit{bias} vector of $\pi$, which intuitively quantifies the difference in accumulated rewards when starting in different states. Accordingly, we are sometimes interested in the \textit{span} of the bias vector $\lambda$ of an optimal policy defined as
$\rm{span}(\lambda):=\max_{s\in\cS^\circ} \lambda(s) - \min_{s'\in \cS^\circ} \lambda(s')$.
In the following we assume that rewards are bounded in $[0,1]$, which implies that $\rm{span}(\lambda)$ is upper bounded by $D$, cf.~\cite{jaorau,regal}.
\smallskip

\textbf{Problem setting.}
Given a finite set of models $\Phi$ (not necessarily containing a Markov model), we want to construct a strategy that performs as well as the algorithm that knows the underlying true Markov model $\phi^\circ$, including its 
rewards and transition probabilities.
For that purpose we define for the Markov model $\phi^\circ$ the \textit{regret} of any strategy at time~$T$, cf.~\cite{jaorau,regal,maimury}, 
as
\[
   \st \Delta(\phi^\circ,T) := T\rho^*(\phi^\circ)-\sum_{t=1}^T{}r_t\,,
\]
where $r_t$ are the rewards received when following the proposed strategy and
$\rho^*(\phi^\circ)$ is the average optimal reward in $\phi^\circ$, i.e.\
$\rho^*(\phi^\circ) := \rho^*(M(\phi^\circ)) \eqdef \rho(M(\phi^\circ),\pi^*_{\phi^\circ}) := \lim_{T \to \infty} \frac{1}{T}\Esp\big[\sum_{t=1}^T r_t(\pi^*_{\phi^\circ})\big]$
where $r_t(\pi^*_{\phi^\circ})$ are the rewards received when following the optimal policy $\pi^*_{\phi^\circ}$ on $M(\phi^\circ)$.
Note that for weakly communicating MDPs the average optimal reward does not depend on the initial state.

We consider the case when $\Phi$ is finite and the learner has no knowledge of the correct approximation 
errors of each model in $\Phi$. 
Thus, while for each model $\phi\in\Phi$ there is an associated $\veps=\veps(\phi)\geq 0$
which indicates the aggregation error (cf.\ Definition \ref{def:approx} below),
this $\veps$ is unknown to the learner for each model.

We remark that we cannot expect to perform as well as the unknown underlying Markov model, if the model set only provides approximations.
Thus, if the best approximation has error $\veps$ we have to be ready to accept respective error of order $\veps D$ per step, cf.\ the lower bound provided by Theorem \ref{thm:lobo} below.
\smallskip

\textbf{Overview.} We start with explicating our notion of approximation in Section~\ref{sec:approx}, then introduce our algorithm in Section~\ref{sec:alg}, present our regret bounds in Section~\ref{sec:regret}, and conclude with the proofs in Section~\ref{sec:proof}. 

\section{Preliminaries: MDP Approximations}\label{sec:approx}
\textbf{Approximations.}
Before we give the precise notion of \textit{approximation} we are going to use, first note that in our setting the transition probabilities
$p(h'|h,a)$ for any two histories $h,h'\in \cH$ and an action $a$ are well-defined. Then given an arbitrary model~$\phi$ and a state $s'\in\cS_\phi$,
we can define the aggregated transition probabilities $p^\agg(s'|h,a):=\sum_{h':\phi(h')=s'} p(h'|h,a)$.
Note that the true transition probabilities under $\phi^\circ$ are then given by 
$p(s'|s,a):=p^\agg(s'|h,a)$ for $s=\phi^\circ(h)$ and $s'\in \cS^\circ$.

\begin{definition}\label{def:approx}
A model $\phi$ is said to be an $\veps$-approximation of the true model $\phit$ 
if: (i) for all histories $h,h'$ with $\phi(h)=\phi(h')$ and all actions $a$
\begin{eqnarray}
\big|  r(\phit(h),a) -  r(\phit(h'),a) \big|  < \veps   \label{eq:agg-rd1},   \mbox{ and }
\big\|  p(\cdot|\phit(h),a) -  p(\cdot|\phit(h'),a) \big\|_1  < \tfrac{\veps}{2} \label{eq:agg-pd1},  
\end{eqnarray}
and (ii) there is a surjective mapping $\alpha:\cS^\circ\to\cS_\phi$ such that 
for all histories $h$ and all actions~$a$ it holds that
\begin{equation}\label{eq:agg-pd2}
 \sum_{\s' \in \cS_\phi} \Big| p^ \agg(\s' | h,a) -  \sum_{s'\in \cS^\circ:\alpha(s')=\s'} p^\agg(s' | h,a) \Big|  < \tfrac{\veps}{2}.  
\end{equation}
\end{definition}

\setcounter{footnote}{0}
Intuitively, condition \eqref{eq:agg-rd1} 
assures that the approximate model aggregates only histories that are mapped to similar states 
under the true model. Complementary, condition 
\eqref{eq:agg-pd2} guarantees that the state space under the approximate model resembles the true state space.\footnote{
The allowed error in the conditions for the transition probabilities is chosen to be~$\tfrac{\veps}{2}$ so that
the total error with respect to the transition probabilities is $\veps$. This matches the respective condition for MDP approximations 
in Definition \ref{def:mdpapprox}, cf.\ also Section~\ref{sec:err-agg}. 
}
Note that any model will be an $\veps$-approximation of the underlying true model $\phi^\circ$ for sufficiently large $\veps$.

A particular natural case are approximation models $\phi$ which also satisfy
\[
   \st \forall h, h' \in \mathcal H:\;   \phi^\circ(h) =  \phi^\circ(h') \Longrightarrow  \phi(h) =  \phi(h').
\]
That is, intuitively, states in $\cS^\circ$ are aggregated to meta-states in $\cS_\phi$, and \eqref{eq:agg-pd2} holds trivially.

We may carry over our definition of $\veps$-approximation to MDPs. 
This will turn out useful, since each approximate model can be interpreted as an MDP approximation, cf.\ Section~\ref{sec:err-agg} below.
\begin{definition}\label{def:mdpapprox}
An MDP $\bar{M}=(\bar{\cS},\cA,\bar{r},\bar{p})$ is an \textit{$\veps$-approximation} of another MDP $M=(\cS,\cA,r,p)$ if
there is a surjective function $\alpha:\cS\to\bar{\cS}$ such that for all~$s$~in~$\cS$:
\begin{eqnarray}
 \big| \bar r(\alpha(s),a) -  r(s,a) \big|  < \veps  \label{eq:agg-rl},   \mbox{ and }
 \sum_{\s' \in \bar \cS} \Big| \bar p(\s' | \alpha(s),a)  -  \sum_{s':\alpha(s')=\s'} \!\!\!\!\!\!\!\! p(s' | s,a) \Big|  < \veps.  \label{eq:agg-pl}
\end{eqnarray}
\end{definition}

\textbf{Error Bounds for $\veps$-Approximations.}\label{sub:error-bound-app}
The following is an error bound on the error made by an $\veps$-approximation.
It generalizes bounds of \cite{or-alt07} from ergodic to communicating MDPs.
For a proof see Appendix~\ref{app:proofthm1}.

\begin{theorem}\label{thm:aggubo}
Let $M$ be a communicating MDP and $\bar{M}$ be an $\veps$-approximation of $M$.
Then 
\[
    \big| \rho^*(M) -  \rho^*(\bar M) \big|  \;\leq\; \veps\,( D(M) + 1).
\]
\end{theorem}

The following is a matching lower bound on the error by aggregation.
This is an improvement over the results in \cite{or-alt07}, which 
only showed that the error approaches 1 when the diameter goes to infinity.

\begin{theorem}\label{thm:lobo}
For each $\veps>0$ and each $2<D<\frac4\veps$
there are MDPs $M$, $\bar M$ such that $\bar M$ is an $\veps$-approximation of $M$, 
$M$ has diameter $D(M)=D$,
and 
\[
	|\rho^*(M) - \rho^*(\bar M)| > \tfrac{1}{56} \veps D(M).
\]
\end{theorem}
\begin{proof}
Consider the MDP $M$ shown in Figure \ref{fig:lobo} (left), where the (deterministic) reward in states $s_0$, $s_0'$ is 0 and 1 in state $s_1$.
\begin{figure}[ht!]
	\centering
		\scalebox{0.1}{\includegraphics{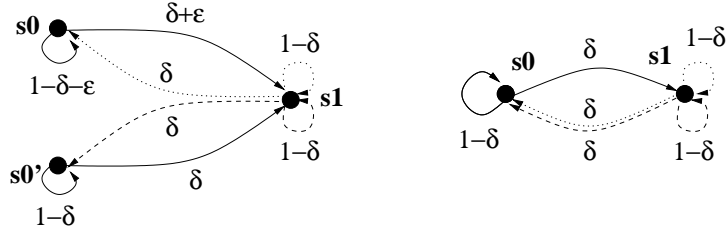}}
	\caption{\label{fig:lobo} The MDPs $M$ (left) and $\bar M$ (right) in the proof of Theorem \ref{thm:lobo}. Solid, dashed, and dotted arrows indicate different actions.}
\end{figure}
We assume that $0<\varepsilon:=\frac{\veps}{2}<\delta:=\frac2D$. 
Then the diameter $D(M)$ is the expected transition time from $s_0'$ to $s_0$ and equal to $\frac{2}{\delta}=D$.
Aggregating states $s_0,s_0'$ gives the MDP $\bar M$ on the right hand side of Figure \ref{fig:lobo}.
Obviously, $\bar M$ is an $\veps$-approximation of $M$.
It is straightforward to check that the stationary distribution $\mu$ (of the only policy) in $M$ is 
$(\mu(s_0),\mu(s_0'),\mu(s_1))=\left(\frac{\delta}{3\varepsilon+4\delta},\frac{\varepsilon+\delta}{3\varepsilon+4\delta},\frac{2\varepsilon+2\delta}{3\varepsilon+4\delta}\right)$,
while the stationary distribution in $\bar M$ is $(\frac12,\frac12)$. Thus, the difference in average reward
is
\[
   \st \;\;\;\;\quad\quad |\rho^*(M) - \rho^*(\bar M)| = \frac{2\varepsilon+2\delta}{3\varepsilon+4\delta} - \frac12 \;=\; \frac{\varepsilon}{2(3\varepsilon+4\delta)}  \;>\;   \frac{\varepsilon}{14\delta}  =  \tfrac{1}{56} \veps D(M).\quad\qquad  \qed
\]
\end{proof}

Theorems~\ref{thm:aggubo} and \ref{thm:lobo} compare the optimal policies of two different MDPs, however it is straightforward to see from the proofs that
the same error bounds hold when comparing on some MDP $M$ the optimal average reward $\rho^*(M)$ to the average reward when 
applying the optimal policy of an $\veps$-approximation $\bar M$ of $M$.
Thus, when we approximate an MDP~$M$ by an $\veps$-approximation~$\bar M$,
the respective error of the optimal policy of $\bar M$ on~$M$ can be of order $\veps D(M)$ as well. 
Hence, we cannot expect to perform below this error if we only have an $\veps$-approximation of the true model at our disposal.

\section{Algorithm}\label{sec:alg}

%
The \AlgoName~algorithm (shown in detail as Algorithm~\ref{fig:algo}) we propose for the setting introduced in Section~\ref{sec:intro} is a generalization of the \texttt{OMS} algorithm of \cite{oms}. Application of the original \texttt{OMS} algorithm to our setting would not work, since \texttt{OMS} compares
the collected rewards of each model to the reward it would receive if the model were Markov. Models
not giving sufficiently high reward are identified as non-Markov and rejected.
In our case, there may be no Markov model in the set of given models~$\Phi$. Thus, the main difference to
\texttt{OMS} is that \AlgoName~for each model estimates and takes into account the possible
approximation error with respect to a closest Markov model.

\AlgoName~proceeds in episodes $k=1,2,\ldots$, each consisting of several runs $j=1,2,\ldots$.
In each run $j$ of some episode $k$, starting at time $t=t_{kj}$, \AlgoName~chooses a policy $\pi_{kj}$ applying the \textit{optimism in face of uncertainty} principle twice.
\smallskip

\textbf{Plausible models.}
First, \AlgoName~considers for each model $\phi\in\Phi$  a set of \textit{plausible} MDPs $\cM_{t,\phi}$
defined to contain all MDPs with state space $\mathcal S_\phi$ and with
rewards $r^+$ and transition probabilities $p^+$ satisfying
\begin{eqnarray}
   \st \big| r^+(s,a) - \what r_t(s,a)\big| &\leq& \geps(\phi) + \sqrt{\tfrac{\log(48 S_\phi A t^3/\delta)}{2N_{t}(s,a)}}\,,\label{eqn:cond_r}\\
 \st \big\|  p^+(\cdot | s,a)-\what p_t(\cdot | s,a)\big\|_1 &\leq& \geps(\phi) + \sqrt{\tfrac{2S_\phi\log(48 S_\phi A t^3/\delta)}{N_{t}(s,a)}}\,,\quad \label{eqn:cond_p}
\end{eqnarray}
where $\geps(\phi)$ is the estimate for the approximation error of model $\phi$ (cf.\ below), $\what p_t(\cdot | s,a)$ and $\what r_t(s,a)$  are respectively the empirical state-transition probabilities and the mean reward at time $t$ for taking action $a$ in state $s\in \mathcal S_\phi$,  $S_\phi:=|\mathcal S_\phi|$ denotes the number of states under model $\phi$, $A:=|\mathcal A|$ is the number of actions, 
and $N_{t}(s,a)$ is the number of times action $a$ has been chosen in state~$s$ up to time~$t$.
(If $a$ hasn't been chosen in $s$ so far, we set $N_{t}(s,a)$ to 1.) 
The inequalities \eqref{eqn:cond_r} and \eqref{eqn:cond_p} are obviously inspired by Chernov bounds that would hold with high probability in case the respective model~$\phi$ is Markov, cf.\ also Lemma~\ref{lem:chernov} below.
\smallskip

\textbf{Optimistic MDP for each model $\phi$.}
In line~4, the algorithm computes for each model $\phi$ a so-called optimistic MDP $M_t^+(\phi)\in \cM_{t,\phi}$ and an 
associated optimal policy $\pi_{t,\phi}^+$ on $M_t^+(\phi)$ such that the average
reward $\rho(M_t^+(\phi),\pi_{t,\phi}^+)$ is maximized.
This can be done by extended value iteration (EVI)~\cite{jaorau}.
Indeed, if $r^+_t(s,a)$ and $p^+_t(s'|s,a)$ denote the optimistic rewards and transition probabilities of $M_t^+(\phi)$, then EVI computes optimistic state values $\v{u}^+_{t,\phi} = (u^+_{t,\phi}(s))_s \in \Real^{S_\phi}$ such that (cf.~\cite{jaorau})
\begin{equation}\label{eqn:empevi}
\what \rho_t^+\!(\phi) := \min_{s\in\cS_\phi} \Big\{ r^+_t\!(s,\pi_{t,\phi}^+(s))
  +\sum_{s'} p^+_t\!(s'|s,\pi_{t,\phi}^+(s))\, u^+_{t,\phi}(s')-u^+_{t,\phi}(s) \Big\}\,
\end{equation}
is an approximation of $\rho^*(M_t^+(\phi))$, that is, 
\begin{equation}\label{eq:rhostar-approx}
  \what \rho_t^+(\phi) \geq \rho^*(M_t^+(\phi)) - 2/\sqrt{t}.
\end{equation}

\textbf{Optimistic model selection.}
In line~5, \AlgoName~chooses a model $\phi_{kj}\in \Phi$ with corresponding MDP $M_{kj} = M_t^+(\phi_{kj})$
and policy $\pi_{kj} := \pi_{t,\phi_{kj}}^+$ that maximizes the average reward
penalized by the term $\pen(\phi,t)$ defined as
\begin{eqnarray}
\pen(\phi,t) &\eqdef&  2^{-j/2} \Big( \big( \span(\v{u}^+_{t,\phi})\,\sqrt{2S_{\phi}}\, + \tfrac{3}{\sqrt{2}}\big) \sqrt{S_{\phi} A \log\big(\tfrac{48 S_{\phi} A t^3}{\delta}\big)} \label{eq:pen}\\
&& \qquad\quad + \span(\v{u}^+_{t,\phi})\sqrt{2\log(\tfrac{24t^2}{\delta})} \Big)   +  2^{-j}\span(\v{u}^+_{t,\phi})  +   \geps(\phi) \big(\span(\v{u}^+_{t,\phi}) + 3 \big), \nonumber    
\end{eqnarray}
where we define $\span(\v{u}^+_{t,\phi}):= \max_{s\in\cS_{\phi}}u^+_{t,\phi}(s)- \min_{s\in\cS_{\phi}} u^+_{t,\phi}(s)$ to be the empirical value span of the optimistic MDP $M_t^+(\phi)$.
Intuitively, the penalization term is an upper bound on the per-step regret of the model $\phi$ in the run to follow in case $\phi$ is chosen, 
cf.\ eq.~\eqref{eq:r1} in the proof of the main theorem.
Similar to the REGAL algorithm of~\cite{regal}
this shall prefer simpler models (i.e., models having smaller state space and smaller value span) to more complex ones.
\smallskip

\textbf{Termination of runs and episodes.}
The chosen policy $\pi_{kj}$ is then executed until either (i) run $j$ reaches the maximal length of $2^j$ steps,
(ii) episode~$k$ terminates when the number of visits in some state has been doubled (line~12), or
(iii) the executed policy $\pi_{kj}$  does not give sufficiently high rewards (line~9).
That is, at any time $t$ in run $j$ of episode $k$ it is checked whether the total reward in the current run is at least $\ell_{kj}\rho_{kj}-\lob_{kj}(t)$,
where $\ell_{kj} \eqdef t-t_{kj}+1$ is the (current) length of run $j$ in episode $k$, and $\lob_{kj}(t)$ is defined as
\begin{eqnarray}
\lob_{kj}(t) &\eqdef& \big(\span_{kj}^{+}\sqrt{2S_{kj}}+\tfrac{3}{\sqrt{2}}\big)\!\!\sum_{s \in \cS_{{kj}}}\sum_{a\in\cA} \sqrt{v_{kj}(s,a)
		\log\!\big(\tfrac{48 S_{{kj}} A t_{kj}^3}{\delta}\big)}\;  \nonumber \\
&& + \,\, \span_{kj}^{+}\sqrt{2\ell_{kj}\log\big(\tfrac{24t_{kj}^2}{\delta}\big)} + \span_{kj}^{+} + \geps\big(\phi_{kj}) \ell_{kj} (\span_{kj}^{+} + 3 \big),\label{eq:lob}
\end{eqnarray}
where $\span_{kj}^+:=\span(\v{u}^+_{t_{kj},\phi_{kj}})$, $S_{kj} := S_{\phi_{kj}}$, and
$v_{kj}(s,a)$ are the (current) state-action counts of run $j$ in episode $k$.
That way, \AlgoName~assumes each model to be Markov, as long as it performs well. 
We will see that $\lob_{kj}(t)$ can be upper bounded by $\ell_{kj}\pen(\phi_{kj},t_{kj})$, cf.\ eq.~\eqref{eq:r1} below. 
\smallskip

\textbf{Guessing the approximation error.}
The algorithm tries to guess for each model $\phi$ the correct approximation error $\veps(\phi)$.
In the beginning the guessed value $\geps(\phi)$ for each model $\phi \in\Phi$ is set to the precision parameter $\veps_0$, the best possible precision we aim for.
Whenever the reward test fails for a particular model $\phi$, it is likely that $\geps(\phi)$ is too small and it is therefore doubled (line~10).

\begin{algorithm}[t!]
\caption{Optimal Approximate Model Selection (\AlgoName)} \label{fig:algo}
\begin{algorithmic}[1]
\INPUT set of models $\Phi$, confidence parameter $\delta\in(0,1)$, precision parameter~$\veps_0\in(0,1)$
\STATE Let $t$ be the current time step, and set $\geps(\phi):=\veps_0$ for all $\phi\in \Phi$.
\FOR{episodes $k=1,2,\ldots$}
  \FOR{runs $j=1,2,\ldots$}
  \STATE $\forall\phi\in\Phi$, use EVI to compute an optimistic MDP $M_t^+(\phi)$ in 
   $\cM_{t,\phi}$ (the set of \textit{plausible} MDPs defined via the confidence intervals \eqref{eqn:cond_p} and \eqref{eqn:cond_r} for the estimates so far), 
   a (near-)optimal policy $\pi_{t,\phi}^+$ on $M_t^+(\phi)$ with approximate average reward~$\what \rho_{t}^+(\phi)$,
   and the empirical value span $\span(\v{u}_{t,\phi}^+)$.
   \STATE Choose model $\phi_{kj}\in  \Phi $ such that
   \begin{equation}\label{eq:modelselection}
       \phi_{kj}=\argmax_{\phi\in\Phi}\big\{\what \rho_{t}^+(\phi)- \pen(\phi,t)\big\}\,.
   \end{equation}
   \STATE Set $t_{kj}:=t$, $\rho_{kj} := \what \rho_{t}^+(\phi_{{kj}})$, $\pi_{kj} := \pi^+_{t,\phi_{kj}}$, 
          and $\cS_{kj}:= \cS_{\phi_{kj}}$.
    \FOR{$2^j$ steps}
       \STATE Choose action $a_t := \pi_{kj}(s_t)$, get reward $r_t$, observe next state $s_{t+1} \in \cS_{kj}$.
       \IF{the total reward collected so far in the current run is less than
       \begin{equation}\label{eq:test}
          \st (t-t_{kj}+1) \rho_{kj} - \lob_{kj}(t),
       \end{equation}}
	    \STATE $\geps({\phi_{kj}}) := 2 \geps({\phi_{kj})}$
	   \STATE Terminate current episode.
       \ELSIF{$\sum_{j'=1}^j v_{kj'}(s_t,a_t)= N_{t_{k1}}(s_t,a_t)$}
           \STATE Terminate current episode.
       \ENDIF
    \ENDFOR
   \ENDFOR
\ENDFOR
\end{algorithmic}
\end{algorithm}

\section{Regret Bounds}\label{sec:regret}
The following upper bound on the regret of \AlgoName\ is the main result of this paper.

\begin{theorem}\label{thm:mainFinite}
There are $c_1,c_2,c_3\in\mathbb{R}$ such that in each learning problem 
where the learner is given access to a set of models $\Phi$ not necessarily containing the true model $\phi^\circ$,
the regret $\Delta(\phi^\circ,T)$ of \AlgoName~(with parameters $\delta$, $\veps_0$) with respect to
the true model~$\phi^\circ$ after any $T\geq SA$ steps is upper bounded by
\begin{eqnarray*}
 c_1\cdot DSA (\log(\tfrac{1}{\veps_0})\log T + \log^2 T) + c_2\cdot D \max\{\veps_0,\veps(\phi)\} T \\ 
     + c_3 \cdot \big(DS_\phi \sqrt{S} A \log^{3/2} (\tfrac{T}{\delta})  + \sqrt{|\Phi|\log(\tfrac{1}{\veps_0})\log T} \big) \sqrt{T}
\end{eqnarray*}
with probability at least $1-\delta$, 
where $\phi\in\Phi$ is an $\veps(\phi)$-approximation of the true underlying Markov model $\phi^\circ$,
$D:=D(\phi^\circ)$, and $S:=\sum_{\phi\in \Phi} S_\phi$.
\end{theorem}

\noindent
As already mentioned, by Theorem~\ref{thm:lobo} the second term in the regret bound is unavoidable when only considering models in $\Phi$.
Note that Theorem~\ref{thm:mainFinite} holds for \textit{all} models $\phi\in\Phi$. For the best possible bound there is a payoff between the size $S_\phi$ of the approximate model and its precision $\veps(\phi)$.

When the learner knows that $\Phi$ contains a Markov model $\phi^\circ$, the original \texttt{OMS} algorithm of~\cite{oms} can be employed. In case when the total number $S=\sum_\phi S_\phi$ of states over all models is large, i.e., $S >D^2 |\Phi| S^\circ$, we can improve on the state space dependence 
of the regret bound given in \cite{oms} as follows.
The proof (found in Appendix~\ref{app:stateimprovement}) is a simple modification of the analysis in~\cite{oms} that exploits that by \eqref{eq:modelselection} the selected models cannot have arbitrarily large state space.

\begin{theorem}\label{thm:stateimprovement}
If $\Phi$ contains a Markov model $\phi^\circ$, with probability at least $1-\delta$ the regret of \texttt{OMS} is bounded by
$\tilde{O}(D^2 {S^\circ}^{3/2} A \sqrt{|\Phi|T} )$.
\end{theorem}

\textbf{Discussion.}
Unfortunately, while the bound in Theorem~\ref{thm:mainFinite} is optimal with respect to the dependence
on the horizon $T$, the improvement on the state space dependence that we could achieve in Theorem~\ref{thm:stateimprovement}
for \texttt{OMS} is not as straightforward for \AlgoName\ and remains an open question just as the optimality of the bound with respect to the other appearing parameters. We note that this is still an open question even for learning in MDPs (without additionally selecting the state representation) as well, cf.~\cite{jaorau}.

Another direction for future work is the extension to the case when the underlying true MDP has continuous state space.
In this setting, the models have the natural interpretation of being discretizations of the original state space.
This could also give improvements over current regret bounds for continuous reinforcement learning as given in~\cite{uccrl}.
Of course, the most challenging goal remains to generate suitable state representation models algorithmically instead of assuming them to be given, cf.~\cite{frl}.
However, at the moment it is not even clear how to deal with the case when an infinite set of models is given.

\section{Proof of Theorem \ref{thm:mainFinite}}\label{sec:proof}

The proof is divided into three parts and follows the lines of~\cite{oms}, now taking into account the necessary modifications to deal with the approximation error. 
First, in Section~\ref{sec:err-agg} we deal with the 
error of $\veps$-approximations. Then in Section~\ref{sub:test}, we show that all state-representation models $\phi$
which are an $\veps(\phi)$-approximation of a Markov model pass the test in \eqref{eq:test}
on the rewards collected so far with high probability, provided that the estimate $\geps(\phi)\geq \veps(\phi)$.
Finally, in Section~\ref{sub:regret} we use this result to derive the regret bound of Theorem~\ref{thm:mainFinite}.

\subsection{Error Bounds for $\veps$-Approximate Models}\label{sec:err-agg}

We start with some observations about the empirical rewards and transition probabilities our algorithm calculates and employs for each model $\phi$.
While the estimated rewards $\what r$ and transition probabilities $\what p$ used by the algorithm do in general not correspond to some underlying
true values, the expectation values of $\what r$ and $\what p$ are still well-defined, given the history $h\in\cH$ so far.
Indeed, consider some $h\in\cH$ with $\phi(h)=\s\in \cS_\phi$,  $\phit(h)=s\in \cS^\circ$,  and an action $a$, and assume that the estimates $\what r(\s,a)$ and $\what p(\cdot|\s,a)$ are 
calculated from samples when action~$a$ was chosen after histories $h_1, h_2,\ldots, h_n \in \cH$ that are mapped to the same state~$\s$ by $\phi$. (In the following, we will denote the states of an approximation~$\phi$
by variables with dot, such as $\s$, $\s'$, etc., and states in the state space $\cS^\circ$ of the true Markov model~$\phi^\circ$ without a dot, such as $s$, $s'$, etc.)
Since rewards and transition probabilities are well-defined under $\phi^\circ$, we have
\begin{eqnarray}
      \E[\what r(\s,a)] &=& \tfrac1n\sum_{i=1}^n r(\phi^\circ(h_i),a),  \mbox{ and } 
      \E[\what p(\s'|\s,a)] =
             \tfrac1n\sum_{i=1}^n \sum_{h':\phi(h')=\s'} \!\!\!\! p(h'|h_i,a). \qquad \label{eq:what-p}\label{eq:what-r} 
\end{eqnarray}

Since $\phi$ maps the histories $h,h_1,\ldots,h_n$ to the same state $\s\in\cS_\phi$, 
the rewards and transition probabilities in the states $\phi^\circ(h),\phi^\circ(h_1),\ldots,\phi^\circ(h_n)$ of the true underlying MDP 
are $\veps$-close, cf.\ \eqref{eq:agg-rd1}.
It follows that for $s=\phi^\circ(h)$ and $\s=\phi(h)$
\begin{eqnarray}
      \Big| \E[\what r(\s,a)] - r(s,a) \Big|  
		    =   \Big|  \tfrac1n\sum_{i=1}^n \big(r(\phi^\circ(h_i),a) - r(\phi^\circ(h),a)\big) \Big|  < \veps(\phi). \label{eq:model-r}
\end{eqnarray}
For the transition probabilities we have by \eqref{eq:agg-pd2} for $i=1,\ldots,n$
\begin{equation}\label{eq:bias3a}
      \sum_{\s' \in \cS_\phi} \Big| p^ \agg(\s' | h_i,a) -  \sum_{s'\in \cS^\circ:\alpha(s')=\s'} p^\agg(s' | h_i,a) \Big|  < \tfrac{\veps(\phi)}{2}. 
\end{equation}
Further, all $h_i$ as well as $h$ are mapped to $\s$ by $\phi$ so that according to \eqref{eq:agg-pd1} and recalling that $s=\phi^\circ(h)$ we have for $i=1,\ldots,n$
\begin{eqnarray}
\sum_{\s'\in\cS_\phi}  \Big| \sum_{s'\in \cS^\circ:\alpha(s')=\s'} p^\agg(s'| h_i,a) - \sum_{s'\in \cS^\circ:\alpha(s')=\s'} p(s'| s,a) \Big|  \nonumber \\
  \leq  \sum_{s'\in\cS^\circ} \big| p^\agg(s'| h_i,a) -  p(s'| s,a)\big|  < \tfrac{\veps(\phi)}{2}. \label{eq:bias3d} 
\end{eqnarray}
By \eqref{eq:bias3a} and \eqref{eq:bias3d} for  $i=1,\ldots,n$
\begin{equation}\label{eq:bias3e}
   \sum_{\s'\in\cS_\phi}  \Big| p^\agg(\s'| h_i,a) - \sum_{s'\in \cS^\circ:\alpha(s')=\s'} p(s'| s,a) \Big|  < \veps(\phi), 
\end{equation}
so that from \eqref{eq:what-p} and \eqref{eq:bias3e} we can finally bound 
\begin{eqnarray}
     \lefteqn{  \sum_{\s' \in \cS_\phi} \Big| \E[ \what p(\s' | \s,a)] -  \sum_{s'\in \cS^\circ:\alpha(s')=\s'} p(s' | s,a) \Big|  }  \nonumber \\
		    &\leq&   \tfrac1n\sum_{i=1}^n  \sum_{\s' \in \cS_\phi} \Big| p^\agg(\s'|h_i,a) -  \sum_{s'\in \cS^\circ:\alpha(s')=\s'} p(s' | s,a)  \Big|  \,<\, \veps(\phi). \label{eq:model-p} 
\end{eqnarray}
Thus, according to \eqref{eq:model-r} and \eqref{eq:model-p} the $\veps$-approximate model $\phi$ gives rise to an MDP $\bar{M}$ on $\cS_{\phi}$ with rewards $\bar{r}(\s,a):=\E[\what r(\s,a)]$ and transition probabilities $\bar{p}(\s' | \s,a):=\E[ \what p(\s' | \s,a)]$ that is an $\veps$-approximation of the true MDP $M(\phi^\circ)$. Note that  $\bar{M}$ actually depends on the history so far.

The following lemma gives some basic confidence intervals for the estimated rewards and transition probabilities.
For a proof sketch see Appendix~\ref{app:proof-chernov}.

\begin{lemma}\label{lem:chernov} 
Let $t$ be an arbitrary time step and $\phi \in \Phi$ be the model employed at step~$t$.
Then the estimated rewards $\what{r}$ and transition probabilities $\what{p}$ satisfy for all $\s,\s'\in\cS_\phi$ and all $a\in\cA$
\begin{eqnarray*}
 \st \what{r}(\s,a) - \E[\what{r}(\s,a)] &\leq& \sqrt{\tfrac{\log(48 S_\phi A t^3/\delta)}{N_t(\s,a)}}, \\
 \st \Big\|\what{p}(\cdot|\s,a) - \E[\what{p}(\cdot|\s,a)]\Big\|_1  &\leq& \sqrt{\tfrac{2 S_\phi \log(48 S_\phi A t^3 /\delta)}{N_t(\s,a)}} ,
\end{eqnarray*}
each with probability at least $1-\frac{\delta}{24t^2}$.
\end{lemma}

The following is a consequence of Theorem~\ref{thm:aggubo}, see Appendix~\ref{app:lem3} for a detailed proof.
\begin{lemma}\label{lem:3}
Let $\phi^\circ$ be the underlying true Markov model leading to MDP $M=(\cS^\circ,\cA,r,p)$,
and $\phi$ be an $\veps$-approximation of $\phi^\circ$.
Assume that the confidence intervals given in Lemma~\ref{lem:chernov} hold at step $t$ for all states $\s,\s'\in\cS_\phi$ and all actions~$a$.
Then the optimistic average reward $\what\rho^+_{t}(\phi)$ defined in \eqref{eqn:empevi} satisfies
\[ 
    \what\rho^+_{t}(\phi)  \;\geq\;  \rho^*(M) - \veps(D(M)+1)  -  \tfrac{2}{\sqrt{t}}.
\]
\end{lemma}

\subsection{Approximate Markov models pass the test in \eqref{eq:test}}\label{sub:test}

Assume that the model $\phi_{kj} \in \Phi$ employed in run~$j$ of episode~$k$ is an $\veps_{kj}:=\veps(\phi_{kj})$-approximation of the true Markov model.
We are going to show that $\phi_{kj}$ will pass the test \eqref{eq:test} on the collected rewards with high probability at any step~$t$,
provided that $\geps_{kj}:=\geps(\phi_{kj}) \geq\veps_{kj}$.

\begin{lemma}\label{lem:error-t}
For each step $t$ in some run~$j$ of some episode~$k$, given that $t_{kj}=t'$ the chosen model $\phi_{kj}$ passes the test in \eqref{eq:test} at step $t$ with probability at least 
$1-\frac{\delta}{6t'^2}$ whenever $\geps_{kj}(\phi_{kj})\geq \veps(\phi_{kj})$.
\end{lemma}

\begin{proof}
In the following, $\s_\tau:=\phi_{kj}(h_\tau)$ and $s_\tau:=\phi^\circ(h_\tau)$ are the states at time step $\tau$ under model $\phi_{kj}$ and the true Markov model $\phi^\circ$, respectively. 
\smallskip

\textbf{Initial decomposition.}
First note that at time $t$ when the test is performed, we have $\sum_{\s \in \cS_{{kj}}}\sum_{a\in\cA} v_{kj}(\s,a) = \ell_{kj} = t-t'+1$, so that
 \begin{eqnarray}\label{eq:decomp}
    \ell_{kj}\rho_{kj} - \sum_{\tau=t'}^t r_\tau 
    = 
    \sum_{\s \in \cS_{{kj}}}\sum_{a\in\cA} v_{kj}(\s,a)\Big(\rho_{kj} - \what{r}_{t':t}(\s,a)\Big)\,,\nonumber
 \end{eqnarray}
where $\what{r}_{t':t}(\s,a)$ is the empirical average reward collected for choosing $a$ in $\s$
from time $t'$ to the current time~$t$ in run $j$ of episode $k$.

Let $r_{kj}^+(\s,a)$ be the rewards and  $p_{kj}^+(\cdot|\s,a)$ the transition probabilities 
of the optimistic model $M_{t_{kj}}^+(\phi_{kj})$. 
Noting that $v_{kj}(\s,a)=0$ when $a\neq \pi_{kj}(\s)$, we get 
\begin{eqnarray} 
\ell_{kj}\rho_{kj} - { \sum_{\tau=t'}^t r_\tau}
 &=&  \sum_{\s,a}v_{kj}(\s,a)\big(\what \rho^+_{kj}(\phi_{kj})-r_{kj}^+(\s,a)\big)  \label{eqn:firstdecompo1}\\
&&+\sum_{\s,a}v_{kj}(\s,a)\big(r_{kj}^+(\s,a)-\what{r}_{t':t}(\s,a)\big).            \label{eqn:firstdecompo}
\end{eqnarray}
We continue bounding the two terms \eqref{eqn:firstdecompo1} and \eqref{eqn:firstdecompo} separately.
\smallskip

\textbf{Bounding the reward term \eqref{eqn:firstdecompo}.}
Recall that $r(s,a)$ is the mean reward for choosing $a$ in $s$ in the true Markov model $\phi^\circ$. 
Then we have at each time step $\tau=t',\ldots,t$ 
\begin{eqnarray}
 \lefteqn{ r^+_{kj}(\s_\tau,a) -  \what{r}_{t':t}(\s_\tau,a)  = \big(r^+_{kj}(\s_\tau,a) - \what r_{t'}(\s_\tau,a)\big) } \nonumber\\
			&&   + \big(\what r_{t'}(\s_\tau,a) -  \E[\what r_{t'}(\s_\tau,a)]\big)  + \big(  \E[\what r_{t'}(\s_\tau,a)]  -   r(s_\tau,a)\big)  \nonumber\\
   &&    +   \big(r(s_\tau,a) - \E[\what{r}_{t':t}(\s_\tau,a)]\big)   +  \big(\E[\what{r}_{t':t}(\s_\tau,a)]  -  \what{r}_{t':t}(\s_\tau,a)\big)\  \nonumber\\
&\leq&   \st \geps_{kj} + 2 \sqrt{\frac{\log(48 S_{{kj}} A t'^3/\delta)}{2N_{t'}(\s,a)}} + 2 \veps_{kj} 
	  +  \sqrt{\frac{\log(48 S_{{kj}} A t'^3/\delta)}{2v_{kj}(\s,a)}},   \label{eq:dec}
\end{eqnarray}
where we bounded the first term in the decomposition by~\eqref{eqn:cond_r},
the second term by Lemma~\ref{lem:chernov},
the third and fourth according to \eqref{eq:model-r},
and the fifth by an equivalent to Lemma~\ref{lem:chernov} for the rewards collected so far in the current run.
In summary, with probability at least $1-\frac{\delta}{12t'^2}$ we can bound \eqref{eqn:firstdecompo} as
\begin{eqnarray}
\hspace{-0.5cm}\sum_{\s,a}\! v_{kj}(\s,a)\big(r^+_{kj}(\s,a) - \what{r}_{t':t}(\s,a)\big)   
\!\leq  3\geps_{kj} \ell_{kj} \!+\!\! \tfrac{3}{\sqrt{2}}\!\sum_{\s,a}\!\! \sqrt{v_{kj}(\s,a)\log\!\big(\!\tfrac{48 S_{{kj}} A t'^3}{\delta}\!\big)}, \label{eq:2.1}
\end{eqnarray}
where we used the assumption that $\geps_{kj}\geq \veps_{kj}$ as well as $v_{kj}(\s,a) \leq N_{t'}(\s,a)$.
\smallskip

\textbf{Bounding the bias term \eqref{eqn:firstdecompo1}.}  
First, notice that we can use \eqref{eqn:empevi} to bound
\[
   \sum_{\s,a}v_{kj}(\s,a) \big( \what \rho^+_{kj}(\phi_{kj})-r_{kj}^+(\s,a)\big) 
      \leq  \sum_{\s,a}v_{kj}(\s,a) \Big(\! \sum_{\s'}  p^+_{kj}(\s'|\s,a)\, u^+_{kj}(\s')-u^+_{kj}(\s)\! \Big),
\]
where $u^+_{kj}(\s):= u_{t_{kj},\phi_{kj}}^+(\s)$ are the state values given by EVI.
Further, since the transition probabilities $p_{kj}^+(\cdot|\s,a)$ sum to~$1$, this
is invariant under a translation of the vector $\v{u}_{kj}^{+}$.
In particular, defining
   $w_{kj}(\s) \eqdef u_{kj}^+(\s) - \tfrac12 \big(\min_{\s\in \cS_{{kj}}} u_{kj}^+(\s) + \max_{\s\in \cS_{{kj}}} u_{kj}^+(\s)\big)$, 
so that $\|\v{w}_{kj}\|_\infty = \span_{kj}^+/2$, 
we can replace $\v{u}_{kj}^{+}$ with $\v{w}_{kj}$, and \eqref{eqn:firstdecompo1} can be bounded as \vspace{-1mm}
\begin{eqnarray}
\lefteqn{ \sum_{\s,a}v_{kj}(\s,a)\big(\what \rho^+_{kj}(\phi_{kj})-r_{kj}^+(\s,a)\big) } \nonumber  \\ 
&\leq&\sum_{\s,a} \sum_{\tau=t'}^t\mathds{1}\big\{(\s_\tau,a_\tau)=(\s,a)\big\}  \Big( \sum_{\s'\in\cS_{kj}}  p^+_{kj}(\s'|\s_\tau,a)\, w_{kj}(\s')-w_{kj}(\s_\tau) \Big). \label{eq:bias-sum}
\end{eqnarray}
Now we decompose for each time step $\tau=t',\ldots,t$ 
\begin{eqnarray}
 \lefteqn{ \sum_{\s'\in\cS_{kj}} p^+_{kj}(\s'|\s_\tau,a)\, w_{kj}(\s') - w_{kj}(\s_\tau) =  } \nonumber \\
&&  \sum_{\s'\in\cS_{kj}} \Big(  p^+_{kj}(\s'|\s_\tau,a) - \what p_{t'}(\s'|\s_\tau,a) \Big) \, w_{kj}(\s')     \label{eq:bias1}\\
&&  + \sum_{\s'\in\cS_{kj}} \Big(   \what p_{t'}(\s'|\s_\tau,a)  -  \E[\what p_{t'}(\s'|\s_\tau,a)]\Big) \, w_{kj}(\s')   \label{eq:bias2}\\
&&  + \sum_{\s'\in\cS_{kj}} \Big( \E[\what p_{t'}(\s'|\s_\tau,a)]  - \sum_{s': \alpha(s')=\s'}  p( s' | s_\tau,a) \Big) \, w_{kj}(\s')  \label{eq:bias3}\\
&&  + \sum_{\s'\in\cS_{kj}}\sum_{s': \alpha(s')=\s'}   p( s' | s_\tau,a) \,  w_{kj}(\s')  -   w_{kj}(\s_\tau)   \label{eq:bias4}
\end{eqnarray}
and continue bounding each of these terms individually.
\smallskip

\textbf{Bounding \eqref{eq:bias1}:} Using $\| \v{w}_{kj} \|_\infty = \span_{kj}^+/2$,  \eqref{eq:bias1} is bounded according to \eqref{eqn:cond_p} as 
\begin{eqnarray}
 \sum_{\s'\in\cS_{kj}} \!\!\!\! \big(  p^+_{kj}(\s'|\s_\tau,a) - \what p_{t'}(\s'|\s_\tau,a)\! \big)  w_{kj}(\s')
	\leq \big\| p^+_{kj}(\cdot|\s_\tau,a) - \what p_{t'}(\cdot|\s_\tau,a) \big\|_1 \| \v{w}_{kj} \|_\infty \nonumber\\
      \leq \tfrac{\geps_{kj} \span_{kj}^+}{2} + \tfrac{\span_{kj}^+}{2}\sqrt{\tfrac{2S_{kj}\log(48 S_{kj} A t'^3/\delta)}{N_{t'}(s,a)}}. \qquad\qquad \label{eq:bias1a}
\end{eqnarray}

\textbf{Bounding \eqref{eq:bias2}:} Similarly, by Lemma~\ref{lem:chernov} 
with probability at least $1-\frac{\delta}{24t'^2}$ we can bound \eqref{eq:bias2} at all time steps $\tau$ as
\begin{eqnarray}
 \sum_{\s'\in\cS_{kj}}\!\!\! \Big(   \what p_{t'}(\s'|\s_\tau,a)  -  \E[\what p_{t'}(\s'|\s_\tau,a)]\Big) \, w_{kj}(\s') 
        \leq \tfrac{\span_{kj}^+}{2}\sqrt{\tfrac{2S_{kj}\log(48 S_{kj} A t'^3/\delta)}{N_{t'}(s,a)}} .  \label{eq:bias2a}
\end{eqnarray}

\textbf{Bounding \eqref{eq:bias3}:} 
By \eqref{eq:model-p} and using that $\| \v{w}_{kj} \|_\infty = \span_{kj}^+/2$, we can bound~\eqref{eq:bias3} by
\begin{equation}\label{eq:bias3f}
 \sum_{\s'\in\cS_{kj}} \Big( \E[\what p_{t'}(\s'|\s_\tau,a)]  - \sum_{s': \alpha(s')=\s'}  p( s' | s_\tau,a) \Big) \, w_{kj}(\s') < \tfrac{\veps_{kj} \span_{kj}^+}{2}.
\end{equation}

\textbf{Bounding \eqref{eq:bias4}:} We set $w'(s):=w_{kj}(\alpha(s))$ for $s\in \cS^\circ$ and rewrite \eqref{eq:bias4} as
\begin{eqnarray}\label{eq:ano}
  \sum_{\s'\in\cS_{kj}}\sum_{s': \alpha(s')=\s'} \!\!\!\!\!  p( s' | s_\tau,a) \,  w_{kj}(\s')  -   w_{kj}(\s_\tau) 
	  = \sum_{s'\in\cS^\circ}   p( s' | s_\tau,a) \,  w'(s')  -   w'(s_\tau).\quad\,
\end{eqnarray}
Summing this term over all steps $\tau=t',\ldots,t$, we can rewrite the sum as a martingale difference sequence, so that
Azuma-Hoeffding's inequality (e.g., Lemma~10 of \cite{jaorau}) yields that with probability at least $1-\frac{\delta}{24t'^3}$ 
\begin{eqnarray}
\sum_{\tau=t'}^{t}  \sum_{s'\in\cS^\circ}  p( s' | s_\tau,a) \, w'(s')  -   w'(s_\tau)  
 =  \sum_{\tau=t'}^{t}   \Big( \sum_{s'} {p}(s'|s_\tau,a)\,w'(s') - w'(s_{\tau+1})\Big) \nonumber  \\
      +  w'(s_{t+1})  - w'(s_{t'}) 
 \;\leq\;  \span_{kj}^{+}\sqrt{2\ell_{kj}\log(\tfrac{24t'^3}{\delta})}\,  +  \span_{kj}^{+} \,, \qquad \label{eq:bias4a}
\end{eqnarray}
since the sequence $X_\tau := \sum_{s'} {p}(s'|s_\tau,a)\, w'(s') - w'(s_{\tau+1})$  
is a martingale difference sequence with $|X_t| \leq \span_{kj}^{+}$. 
\smallskip

\textbf{Wrap-up.}
Summing over the steps $\tau=t',\ldots,t$, we get from \eqref{eq:bias-sum}, \eqref{eq:bias4}, \eqref{eq:bias1a}, \eqref{eq:bias2a}, \eqref{eq:bias3f}, \eqref{eq:ano}, and \eqref{eq:bias4a} that with probability at least $1-\frac{\delta}{12t'^2}$
\begin{eqnarray}
\lefteqn{ \sum_{\s,a}v_{kj}(\s,a)\big(\what \rho^+_{kj}(\phi_{kj})-r_{kj}^+(\s,a)\big) \;\leq\;   \geps_{kj} \span_{kj}^+ \ell_{kj} } \nonumber \\
 &&+  \span_{kj}^+ \sum_{\s,a} \sqrt{2 v_{kj}(\s,a){\,S_{kj}\log\big(\tfrac{48 S_{kj} A t'^3}{\delta}\big)}}  +   \span_{kj}^{+}\sqrt{2\ell_{kj}\log\big(\tfrac{24t'^2}{\delta}\big)}\,  +  \span_{kj}^{+} ,\quad \label{eq:bias-summ}
\end{eqnarray} 
using that $v_{kj}(\s,a)\leq N_{t'}(\s,a)$ and the assumption that $\veps_{kj}\leq \geps_{kj}$. 
Combining \eqref{eqn:firstdecompo}, \eqref{eq:2.1}, and \eqref{eq:bias-summ} gives the claimed lemma.  \qed
\end{proof}

Summing Lemma~\ref{lem:error-t} over all episodes gives the following lemma, for a detailed proof see Appendix~\ref{app:error}.
\begin{lemma}\label{lem:error}
With probability at least $1-\delta$, for all runs~$j$ of all episodes~$k$ the chosen model $\phi_{kj}$ passes all tests, provided that
$\geps_{kj}(\phi_{kj})\geq \veps(\phi_{kj})$.
\end{lemma}

\subsection{Preliminaries for the proof of Theorem~\ref{thm:mainFinite}}\label{sub:regret}
We start with some auxiliary results for the proof of Theorem~\ref{thm:mainFinite}.
Lemma \ref{lem:bd} bounds the bias span of the optimistic policy, 
Lemma~\ref{lem:geps} deals with the estimated precision of $\phi_{kj}$,
and Lemma~\ref{lem:episodes} provides a bound for the number of episodes.
For proofs see Appendix~\ref{app:bd}, \ref{app:geps}, and \ref{app:episodes}.
\begin{lemma}\label{lem:bd}
Assume that the confidence intervals given in Lemma~\ref{lem:chernov} hold at some step~$t$ for all states $\s\in\cS_\phi$ and all actions $a$. 
Then for each $\phi$, the set of plausible MDPs $\mathcal{M}_{t,\phi}$ contains an MDP $\wt{M}$ with diameter $D(\wt{M})$ upper bounded by the true diameter $D$,
provided that $\geps(\phi)\geq\veps(\phi)$. Consequently, the respective bias span $\span(\v{u}_{t,\phi}^+)$ is bounded by $D$ as well.
\end{lemma}
\begin{lemma}\label{lem:geps}
If all chosen models $\phi_{kj}$ pass
all tests in run~$j$ of episode~$k$ whenever $\geps(\phi_{kj})\geq \veps(\phi_{kj})$, then
$\geps(\phi) \leq \max\{\veps_0, 2\veps(\phi)\}$ always holds for all models $\phi$.
\end{lemma}
\begin{lemma}\label{lem:episodes}
Assume that all chosen models $\phi_{kj}$ pass all tests in run~$j$ of episode~$k$ whenever $\geps(\phi_{kj})\geq \veps(\phi_{kj})$.
Then the number of episodes $K_T$ after any $T\geq SA$ steps is upper bounded as
$K_T \,\leq\,  S A \log_2\!\big(\tfrac{2T}{S A}\big) + \sum_{\phi:\veps(\phi)>\veps_0}\log_2 \big(\tfrac{\veps(\phi)}{\veps_0}\big)$.
\end{lemma}

\subsection{Bounding the regret (Proof of Theorem~\ref{thm:mainFinite})}
Now we can finally turn to showing the regret bound of Theorem~\ref{thm:mainFinite}. We will assume that all chosen models $\phi_{kj}$ pass all tests in run~$j$ of episode~$k$ 
whenever $\geps(\phi_{kj})\geq \veps(\phi_{kj})$. According to Lemma \ref{lem:error} this holds with probability at least $1-\delta$.

Let $\phi_{kj} \in \Phi$ be the model that has been chosen at time $t_{kj}$, 
and consider the last but one step $t$ of run~$j$ in episode $k$.
The regret $\Delta_{kj}$ of run~$j$ in episode $k$ with respect to $\rho^*:=\rho^*(\phi^\circ)$ is bounded by
\begin{eqnarray*}
\st  \Delta_{kj} &\eqdef& (\ell_{kj}+1)\rho^* - \st \sum_{\tau=t_{kj}}^{t+1} r_\tau 
\;\leq\; 
\ell_{kj}\big(\rho^* - \rho_{kj}\big) + \rho^*  +  \ell_{kj}\rho_{kj} - \sum_{\tau=t_{kj}}^{t}r_\tau \,,
\end{eqnarray*}
where as before $\ell_{kj}:=t-t_{kj}+1$ denotes the length of run~$j$ in episode $k$ up to the considered step $t$.
By assumption the test \eqref{eq:test} on the collected rewards has been passed at step $t$,
so that
\begin{eqnarray}\label{eq:epr}
   \Delta_{kj} \leq \ell_{kj}\big(\rho^* - \rho_{kj}\big) + \rho^*  +  \lob_{kj}(t),
\end{eqnarray}
and we continue bounding the terms of $\lob_{kj}(t)$.
\smallskip

\textbf{Bounding the regret with the penalization term.}
Since we have $v_{kj}(\s,a) \leq N_{t_{k1}}(\s,a)$ for all $\s \in \cS_{{kj}},a\in\cA$
and also $\sum_{\s,a} v_{kj}(\s,a) = \ell_{kj} \leq 2^j$, by Cauchy-Schwarz inequality 
$\sum_{\s,a} \sqrt{v_{kj}(\s,a)} \leq 2^{j/2}\sqrt{S_{{kj}} A }$.
Applying this to the definition~\eqref{eq:lob}
of $\lob_{kj}$, we obtain from \eqref{eq:epr} and by the definition~\eqref{eq:pen} of the penalty term that
\begin{eqnarray}
\Delta_{kj} &\,\leq\,& \ell_{kj}\big(\rho^* - \rho_{kj}\big) + \rho^*    + 2^{j/2}\big(\span_{kj}^{+}\sqrt{2S_{kj}}+\tfrac{3}{\sqrt{2}}\big) \sqrt{S_{{kj}} A \log\big(\tfrac{48 S_{{kj}} A t_{kj}^3}{\delta}\big)}\;  \nonumber \\
&& + \,\,\span_{kj}^{+}\sqrt{2\ell_{kj}\log\big(\tfrac{24t_{kj}^2}{\delta}\big)} + \span_{kj}^{+} +  \geps\big(\phi_{kj}) \ell_{kj} (\span_{kj}^{+} + 3 \big) \nonumber\\
&\leq&  \ell_{kj}\big(\rho^* - \rho_{kj}\big) + \rho^*   +   2^j \pen(\phi_{kj},t_{kj}).  \label{eq:r1}
\end{eqnarray}

\textbf{The key step.}
Now, by definition of the algorithm and Lemma \ref{lem:3}, for any approximate model $\phi$ 
we have  
\begin{eqnarray}
 \rho_{kj} - \pen(\phi_{kj},t_{kj}) &\geq& \what \rho_{t_{kj}}^+(\phi)- \pen(\phi,t_{kj}) \label{eq:key}\\
 &\geq& \rho^* - (D+1) \veps(\phi)  -  \pen(\phi,t_{kj}) - 2t_{kj}^{-1/2},  \nonumber
\end{eqnarray}
or equivalently
$\rho^* - \rho_{kj} + \pen(\phi_{kj},t_{kj}) 
	    \leq  \pen(\phi,t_{kj})  +  (D+1) \veps(\phi) +  2 t_{kj}^{-1/2}$.
Multiplying this inequality with $2^j$ and noting that $\ell_{kj}\leq 2^j$ then gives
\begin{eqnarray*}
\ell_{kj}\big(\rho^* - \rho_{kj}\big)  + 2^j \pen(\phi_{kj},t_{kj}) 
      &\leq&  2^j \pen(\phi,t_{kj})  +  2^{j}(D+1) \veps(\phi) +  2^{j+1} t_{kj}^{-1/2}.  
\end{eqnarray*}
Combining this with \eqref{eq:r1}, we get by application of Lemma~\ref{lem:bd}, i.e., $\span(\v{u}^+_{t_{kj},\phi})  \leq D$,
and the definition of the penalty term~\eqref{eq:pen} that
\begin{eqnarray*}
\Delta_{kj} &\leq& \rho^* +    2^{j/2} \Big( \big( D\,\sqrt{2S_{\phi}}\, + \tfrac{3}{\sqrt{2}}\big) \sqrt{S_{\phi} A \log\big(\tfrac{48 S_{\phi} A t_{kj}^3}{\delta}\big)} +  D\sqrt{2\log(\tfrac{24t_{kj}^2}{\delta})} \Big) \nonumber \\
&& +  D  +   2^j\geps(\phi) \big(D + 3\big)  +  2^{j}(D+1) \veps(\phi) +  2^{j+1} t_{kj}^{-1/2}  \,. \quad \label{eq:dkj}
\end{eqnarray*}
By Lemma~\ref{lem:geps} and using that $2 t_{kj} \geq 2^j$ (so that $2^{j+1} t_{kj}^{-1/2} \leq 2\sqrt{2} \cdot 2^{j/2}$) we get 
\begin{eqnarray}
\Delta_{kj} &\leq& \rho^* +    2^{j/2} \Big( \big( D\,\sqrt{2S_{\phi}}\, + \tfrac{3}{\sqrt{2}}\big) \sqrt{S_{\phi} A \log\big(\tfrac{48 S_{\phi} A t_{kj}^3}{\delta}\big)} + D\sqrt{2\log(\tfrac{24t_{kj}^2}{\delta})} \Big) \nonumber \\
&& +  D  +  \tfrac32\cdot 2^{j}\max\{\veps_0,2\veps(\phi)\} \big(D + \tfrac73\big)  +  2\sqrt{2} \cdot 2^{j/2}.  \quad \label{eq:dkj2}
\end{eqnarray}

\textbf{Summing over runs and episodes.}
Let $J_k$ be the total number of runs in episode $k$, and let $K_T$ be the total number of episodes up to time $T$.
Noting that $t_{kj}\leq T$ and summing \eqref{eq:dkj2} over all runs and episodes gives
\begin{eqnarray*}
\lefteqn{\Delta(\phi^\circ,T)=\sum_{k=1}^{K_T}\sum_{j=1}^{J_k} \Delta_{kj}   \leq  \big(\rho^* + D \big)\!\sum_{k=1}^{K_T} J_k  + \tfrac32\max\{\veps_0,2\veps(\phi)\} \!\big( D + \tfrac73\big)  \sum_{k=1}^{K_T}\sum_{j=1}^{J_k} 2^{j}}   \label{eq:r2} \\
&&+ \bigg(\! \big( D\sqrt{2S_{\phi}}\, + \tfrac{3}{\sqrt{2}}\big) \sqrt{S_{\phi} A \log\!\big(\tfrac{48 S_{\phi} A T^3}{\delta}\big)} + D\sqrt{2\log(\tfrac{24T^2}{\delta})} + 2\sqrt{2} \!\bigg)\! \sum_{k=1}^{K_T}\sum_{j=1}^{J_k} 2^{j/2}.   \nonumber  
\end{eqnarray*}
As shown in Section 5.2 of~\cite{oms}, $\sum_k J_k \leq K_T \log_2(2T/K_T)$, 
$\sum_{k}\sum_{j} 2^{j} \leq 2(T+K_T)$ and 
$\sum_{k}\sum_{j} 2^{j/2} \leq \sqrt{2K_T \log_2(2T/K_T)(T+K_T)}$, 
and we may conclude the proof applying Lemma~\ref{lem:episodes} and some minor simplifications. \qed
\smallskip

\subsubsection*{Acknowledgments.}  
 This research was funded by the Austrian Science Fund (FWF): P~26219-N15, 
 the European Community's  FP7 Program  under grant agreements 
 n$^\circ$\,270327 (CompLACS) and 306638 (SUPREL), the Technion,
the Ministry of Higher Education and Research of France, Nord-Pas-de-Calais Regional Council, and
 FEDER (Contrat de Projets Etat Region CPER  2007-2013).

\bibliographystyle{splncs03}

\newpage

\begin{appendix}
 
\section{Proof of Theorem~\ref{thm:aggubo}}\label{app:proofthm1}
We start with an error bound for approximation, assuming we compare two MDPs over the same state space.

\begin{lemma}\label{lem:1}
Consider a communicating MDP $M=(\cS,\cA,r,p)$, and another MDP ${\bar M}=({\cS},\cA,{\bar r},{\bar p})$ 
over the same state-action space which is an $\veps$-approximation of $M$ 
(for $\alpha={\rm id}$).
Assume that an optimal policy $\pio$ of $M$ is performed on $\bar M$ for $\ell$ steps, 
and let $\bar v^*(s)$ be the number of times state $s$ is visited state among these $\ell$ steps. 
Then
 \[
    \st \ell \rho^*({M}) - \sum_{s\in \cS} \bar v(s) \cdot \bar r(s,\pio(s))  \;<\;   \ell\veps(D + 1)  + D \sqrt{2\ell\log(1/\delta)}\,  +  D 
 \]
with probability at least $1-\delta$, where $D=D(M)$ is the diameter of $M$.
\end{lemma}

\begin{proof}
We abbreviate $r^*(s):=r(s,\pio(s))$ and $p^*(s'|s):=p(s'|s,\pio(s))$, and use $\bar{r}^*(s)$ and  $\bar p^*(s'|s)$ accordingly.
Then 
\begin{eqnarray} 
\lefteqn{\ell\rho^*(M) -  \sum_{s\in \cS}  \bar{v}^*(s) \cdot \bar{r}^*(s)   =    \sum_{s} \bar{v}^*(s) \big(\rho^*(M) - \bar{r}^*(s)\big) }   \nonumber\\
&=&   \sum_{s} \bar{v}^*(s)\big(\rho^*(M)- {r}^*(s) \big)  +   \sum_{s} \bar{v}^*(s)\big( {r}^*(s) - \bar{r}^*(s) \big).   \label{eq:ll1}
\end{eqnarray}
Now, for the first term in \eqref{eq:ll1} we can use the Poisson equation (for the optimal policy $\pio$ on $M$) to replace 
    $\rho^*(M)- {r}^*(s) =  \sum_{s'}  p^*(s'|s)\cdot\lambda^*(s') - \lambda^*(s)$,
writing $\lambda^*:=\lambda_{\pi^*}$ for the bias of $\pi^*$ on $M$.
Concerning the second term in \eqref{eq:ll1} we can use \eqref{eq:agg-rl} and the fact that $\sum_{s} \bar v^*(s)=\ell$. In summary, we get
\begin{eqnarray}
   \lefteqn{\ell\rho^*(M) -  \sum_{s\in \cS} \bar{v}^*(s) \cdot \bar{r}^*(s) 
      <   \sum_{s} \bar{v}^*(s) \Big(  \sum_{s'} p^*(s'|s)\, \lambda^*(s') - \lambda^*(s) \Big)  +  \ell\veps } \qquad\qquad \nonumber \\
      &=& 	\ell\veps +  \sum_{s} \bar{v}^*(s) \Big( \sum_{s'} \bar{p}^*(s'|s)\, \lambda^*(s') - \lambda^*(s) \Big)  \nonumber\\
	     && +   \sum_{s} \bar{v}^*(s) \Big(  \sum_{s'} p^*(s'|s)\, \lambda^*(s')  -  \sum_{s'} \bar{p}^*(s'|s)\, \lambda^*(s')  \Big).
	      \label{eq:ll2}
\end{eqnarray}
By \eqref{eq:agg-pl} and using that $\rm{span}(\lambda^*)\leq D$, the last term in \eqref{eq:ll2} is bounded as
\begin{eqnarray}\label{eq:ll3}
 \sum_{s} \bar v^*(s)   \sum_{s'} \Big( {p}^*(s'|s,a)  - \bar{p}^*(s'|s,a) \Big)  \lambda^*(s')   &<&  \ell \veps D.
\end{eqnarray}
On the other hand, for the second term in \eqref{eq:ll2}, writing $s_\tau$ for the state visited at time step $\tau$ we have
\begin{eqnarray}
\lefteqn{ \sum_{s} \bar v^*(s) \Big( \sum_{s'} \bar{p}^*(s'|s)\,\lambda^*(s') - \lambda^*(s) \Big)
  =  \sum_{\tau=1}^{\ell}  \Big( \sum_{s'} \bar{p}^*(s'|s_\tau)\, \lambda^*(s') - \lambda^*(s_\tau)\Big) }  \qquad\qquad \nonumber\\
  &=&  \sum_{\tau=1}^{\ell}   \Big( \sum_{s'} \bar{p}^*(s'|s_\tau)\, \lambda^*(s') - \lambda^*(s_{\tau+1})\Big)  
      + \lambda^*(s_{\ell+1})  -  \lambda^*(s_{1}).  \label{eq:ll4}
\end{eqnarray}
Now $\lambda^*(s_{\ell+1}) - \lambda^*(s_{1}) \leq \rm{span}(\lambda^*)\leq D$, while
the sequence 
\[
    \st   X_\tau := \sum_{s'} \bar{p}^*(s'|s_\tau)\, \lambda^*(s') - \lambda^*(s_{\tau+1})
\]
is a martingale difference sequence with $|X_t| \leq {D}$.
Thus, an application of Azuma-Hoeffding's inequality (e.g., Lemma 10 of \cite{jaorau}) to \eqref{eq:ll4} yields that
\begin{eqnarray}\label{eq:martingale-ll}
 \sum_{s} \bar v^*(s) \Big( \sum_{s'} \bar{p}^*(s'|s)\,\lambda^*(s') - \lambda^*(s) \Big)  &\leq&  D \sqrt{2\ell\log(1/\delta)}\,  + D 
\end{eqnarray}
with probability higher than $1-\delta$.
Summarizing, \eqref{eq:ll2}, \eqref{eq:ll3}, and \eqref{eq:martingale-ll} give the claimed
\begin{eqnarray*}
 \ell\rho^*(M) -  \sum_{s\in \cS}  \sum_{a\in \cA} \bar{v}(s)^* \cdot \bar{r}^*(s)
&<&    \ell\veps(D + 1)  + D \sqrt{2\ell\log(1/\delta)}\,  +  D.
\end{eqnarray*}
 \qed
\end{proof}

As a corollary to Lemma~\ref{lem:1} we can also bound the approximation error in average reward,
which we will need below to deal with the error of $\veps$-approximate models.
\begin{lemma}\label{lem:2}
Let $M$ be a communicating MDP with optimal policy $\pi^*$, and $\bar M$ an $\veps$-approximation of $M$ over the same state space.
Then
\[
    \big| \rho^*(M) -  \rho^*(\bar M) \big|  \;\leq\;  \big| \rho^*(M) -  \rho(\bar M,\pi^*) \big|  \;\leq\; \veps\, (D(M) + 1).
\]
\end{lemma}
\begin{proof}
Divide the result of Lemma~\ref{lem:1} by $\ell$, choose $\delta=1/\ell$, and let $\ell\to\infty$.
Since the average reward of a policy is no random value, the result holds surely and not just with probability 1. 
\qed
\end{proof}

\subsection{Proof of Theorem~\ref{thm:aggubo}}
The idea is to define a new MDP $M'=(\cS,\cA,r',p')$ on $\cS$ whose rewards~$r'$ and transition probabilities $p'$ are $\veps$-close
to $M$ and that has the same optimal average reward as $\bar M$. 
Thus, for each state $s\in\cS$ and each action $a$ we set $r'(s,a):=\bar r(\alpha(s),a)$ and
\[
     p'(s'|s,a) := \frac{p(s'|s,a)\cdot \bar p(\alpha(s') | \alpha(s),a)}{\sum_{s'':\alpha(s'')=\alpha(s')} p(s'' | s,a)} . 
\]
Note that $p'(\cdot|s,a)$ is indeed a probability distribution over $\cS$, that is, in particular it holds that
\begin{eqnarray*}
     \sum_{s'\in\cS}  p'(s'|s,a) &=&  \sum_{s'\in\cS}\frac{p(s'|s,a)}{\sum_{s'':\alpha(s'')=\alpha(s')} p(s'' | s,a)} \cdot \bar p(\alpha(s') | \alpha(s),a) \\
   &=&  \sum_{\s'\in\bar{\cS}} \sum_{s': \alpha(s')=\s'} \frac{p(s'|s,a)}{\sum_{s'':\alpha(s'')=\s'} p(s'' | s,a)}  \cdot \bar p(\s' | \alpha(s),a) \\
          &=& \sum_{\s'\in\bar{\cS}}   \bar p(\s' | \alpha(s),a) \, = \,  1 .
\end{eqnarray*}
Now by definition, the rewards $r'(s,a)=\bar{r}(\alpha(s),a)$ and aggregated transition probabilities 
\begin{eqnarray*}
  \,\sum_{s':\alpha(s') = \s'}  \!\!\!\! p'(s'|s,a) 
    =   \!\!\!\!\! \sum_{s':\alpha(s') = \s'} \frac{p(s'|s,a) }{\sum_{s'':\alpha(s'')=\s'} p(s'' | s,a)} \cdot \bar p(\s' | \alpha(s),a)  
    =  \bar{p}(\s'|\alpha(s),a)
\end{eqnarray*}
 in $M'$ have the same values for all states $s$ 
that are mapped to the same meta-state by $\alpha$. It follows that $\rho^*(M')=\rho^*(\bar M)$. 

Further by assumption, according to \eqref{eq:agg-rl}
we have 
\[
    |r(s,a)-r'(s,a)| \,=\, |r(s,a)-\bar{r}(\alpha(s),a)|  \,<\,  \veps
\]
and 
\begin{eqnarray*}
  \lefteqn{ \sum_{s'\in\cS} \big|p(s'|s,a) - p'(s'|s,a)\big|  =  \sum_{s'\in\cS} p(s'|s,a) 
                                      \bigg| 1  -  \frac{\bar p(\alpha(s') | \alpha(s),a)}{\sum_{s'':\alpha(s'')=\alpha(s')} p(s'' | s,a)} \bigg|  } \qquad\qquad\\
   &=&  \sum_{\s'\in\bar{\cS}} \sum_{s':\alpha(s')= \s'}  p(s'|s,a) \cdot \bigg|
                                                \frac{ \sum_{s'':\alpha(s'')=\s'} p(s'' | s,a) -  \bar p(\s' | \alpha(s),a)}{\sum_{s'':\alpha(s'')=\s'} p(s'' | s,a)} \bigg|  \\
   &=& \sum_{\s'\in\bar{\cS}} \Big| \! \sum_{s'':\alpha(s'')=\s'} p(s'' | s,a) -  \bar p(\s' | \alpha(s),a) \Big| 
   \;<\;  \veps.
\end{eqnarray*}
Thus, $M'$ is an $\veps$-approximation of $M$ that has the same optimal average reward as $\bar M$ so that
application of Lemma~\ref{lem:2} to $M$ and $M'$ gives the claimed result.\qed

\section{Proof of Lemma~\ref{lem:chernov}}\label{app:proof-chernov}
For any given number of observations $n$ it holds that (cf.~Appendix C.1 of~\cite{jaorau}) for any $\theta>0$
\begin{eqnarray*}
  \st \P\bigg\{ \what{r}(\s,a) - \E[\what{r}(\s,a)] \leq \sqrt{\tfrac{\log(2/\theta)}{n}} \bigg\} < \theta, \\
  \st \P\bigg\{ \Big\|\what{p}(\cdot|\s,a) - \E[\what{p}(\cdot|\s,a)]\Big\|_1  \leq \sqrt{\tfrac{2 S_\phi \log(2 /\theta)}{n}}  \bigg\} < \theta.
\end{eqnarray*}
Choosing suitable values for $\theta$, a union bound over all states $\s$, all actions $a$ and all possible values for $N_t(\s,a)=1,\ldots,t$ shows 
the lemma. \qed

\section{Proof of Lemma~\ref{lem:3}}\label{app:lem3}
Let $\bar{M}$ be the MDP on $\cS_\phi$
whose rewards and transition probabilities are given by the expectation values
 $\E[\what r(\s,a)]$ and $\E[\what p(\s'|\s,a)]$, respectively.
We have already seen in \eqref{eq:model-r} and \eqref{eq:model-p} that $\bar{M}$ is an $\veps$-approximation of the true MDP $M$,
so that by Theorem~\ref{thm:aggubo}  
\begin{equation}\label{eq:no1}
   \big| \rho^*(M) - \rho^*(\bar{M}) \big|  \leq \veps (D(M)+1).
\end{equation}

It remains to deal with the difference between $\rho^*(\bar{M})$ and $\what\rho^+_{t}(\phi)$.
By assumption, the confidence intervals of Lemma~\ref{lem:chernov} hold for all state-action-pairs so that 
$\bar{M}$ is contained in the set of plausible MDPs $\mathcal M_{t,\phi}$ 
(defined via the empirical rewards and transition probabilities $\what r(\s,a)$ and $\what p(\s'|\s,a)$). 
It follows together with \eqref{eq:rhostar-approx} that 
\begin{equation}\label{eq:leq1}
   \what \rho_t^+(\phi) \;\geq\; \rho^*(M_t^+(\phi)) - \tfrac{2}{\sqrt t} \;\geq\; \rho^*(\bar{M}) - \tfrac{2}{\sqrt t}  ,
\end{equation}
which together with \eqref{eq:no1} proves the claimed inequality. \qed

\section{Proof of Lemma~\ref{lem:error}}\label{app:error}
By Lemma~\ref{lem:error-t}, at each step~$t$ of a run~$j$ in an episode~$k$ starting at step~$t_{kj}=t'$ the test is passed with probability at least $1-\frac{\delta}{6t'^2}$. 
Assuming that $t''$ is the last step in that run and setting $\ell:=t''-t'+1$ to be the total number of steps in that run, the test is passed in all steps of the run with error probability bounded by (using that $2t'\geq 2^j \geq \ell$) 
\begin{equation*}
 \ell\cdot \frac{\delta}{6t'^2} \,\leq\, \frac{\ell\delta}{2t'(t' + \ell)}
= \frac{\delta}{2t'}  - \frac{\delta}{2(t'+\ell)}  =  \int_{t'}^{t''} \frac{\delta}{2\tau^2} d\tau \,\leq\, \sum_{\tau=t'}^{t''} \frac{\delta}{2\tau^2}\,.
\end{equation*}
Summing over all episodes and runs shows that the test is passed in all time steps with probability at least $1- \sum_{\tau=1}^\infty \frac{\delta}{2\tau^2} \geq 1-\delta$. \qed

\section{Proof of Lemma~\ref{lem:bd}}\label{app:bd}
We define an MDP $\wt{M}$ on state space $\cS_\phi$ as follows. First let $\beta:\cS_\phi\to\cS$ be an arbitrary mapping that maps states in $\cS_\phi$ to some state in $\cS$ such that $\alpha(\beta(\s))=\s$. Intuitively, $\beta(\s)$ is an arbitrary reference state that is mapped to $\s$ by $\alpha$. Then for $\s,\s'$ in $\cS_\phi$ we set the transition probabilities of  $\wt{M}$ as
\begin{equation}
 \wt p(\s'|\s,a) := \sum_{s':\alpha(s')=\s'} p(s'|\beta(\s),a).
\end{equation}
Then by \eqref{eq:model-p} and Lemma~\ref{lem:chernov} we obtain
\begin{eqnarray*}
 \lefteqn{ \big\| \wt p(\cdot|\s,a) - \what p_t(\cdot|\s,a) \big\|_1  =  \sum_{\s'} \Big| \sum_{s':\alpha(s')=\s'} p(s'|\beta(\s),a)  -  \what p_t(\cdot|\s,a) \Big|  } \\
  &\leq&   \sum_{\s'} \Big( \Big|\!\!\! \sum_{s':\alpha(s')=\s'} \!\!\! p(s'|\beta(\s),a)  - \E[ \what p_t(\s'|\s,a)] \Big|   
                       +   \Big| \E[ \what p_t(\s'|\s,a)] - \what p_t(\s'|\s,a)  \Big| \Big) \\
  &\leq& \veps(\phi)  +  \sqrt{\tfrac{2S_\phi\log(48 S_{\phi} A t^3/\delta)}{N_{t}(s,a)}},
\end{eqnarray*}
showing that $\wt{M}$ is contained in $\mathcal{M}_{t,\phi}$. 
To see that $D(\wt{M})\leq D$, note that $\beta$ maps all transitions in $\wt{M}$ to transitions of the the same or lower probability in the true MDP.
That is, for any $\s,\s'\in \cS_\phi$ it holds that $\wt p(\s'|\s,a) \geq p(\beta(\s')|\beta(\s),a)$.
Thus, each trajectory in $\wt{M}$ can be mapped to a trajectory in the true MDP that cannot have higher probability, which
proves the first claim of the lemma. The second claim follows immediately along the lines of Section~4.3.1 in \cite{jaorau}. \qed

\section{Proof of Lemma~\ref{lem:geps}}\label{app:geps}
By definition of the algorithm, $\geps(\phi)$ for each model $\phi$ has initial value $\veps_0$
and is doubled whenever $\phi$ fails a test. Thus, by assumption if $\veps_0 \leq \veps(\phi)$,
then as soon as $\veps(\phi)\leq\geps(\phi)<2\veps(\phi)$ the value of $\geps(\phi)$ will not change anymore,
and consequently $\geps(\phi)<2\veps(\phi)$ always holds.

On the other hand, if $\veps_0 > \veps(\phi)$ then also $\geps(\phi)\geq \veps(\phi)$ for the initial value $\geps(\phi)=\veps_0$
and again by assumption $\geps(\phi)=\veps_0$ remains unchanged, so that $\geps(\phi)\leq\veps_0$ holds.   \qed

\section{Proof of Lemma~\ref{lem:episodes}}\label{app:episodes}
First recall that an episode is terminated when either the number of visits in some state-action pair $(s,a)$ has been doubled
(line 12 of the algorithm) or when the test on the accumulated rewards has failed (line 9). By assumption, the test is passed 
provided that $\geps(\phi)\geq \veps(\phi)$. If $\veps(\phi)\leq \veps_0$, then $\geps(\phi)\geq \veps(\phi)$ holds trivially. 
Otherwise, $\phi$ will fail the test only $\log_2 \big(\frac{\veps(\phi)}{\veps_0}\big)$  
times until $\geps(\phi)\geq \veps(\phi)$ (after which the test is passed w.h.p.\ and $\geps(\phi)$ remains unchanged by Lemma~\ref{lem:error-t}).
Therefore, the number of episodes terminated due to failure of the test is upper bounded by 
$\sum_{\phi:\veps(\phi)>\veps_0}\log_2 \big(\frac{\veps(\phi)}{\veps_0}\big)$.

For the number of episodes terminated since the number of visits in some state-action pair $(s,a)$ has been doubled,
one can show that it is bounded by  $S A \log_2\!\big(\tfrac{2T}{S A}\big)$, cf.\ Appendix~C.2 of \cite{jaorau} or Section~5.2 of~\cite{oms},
and the lemma follows.\qed

\section{Proof of Theorem~\ref{thm:stateimprovement}}\label{app:stateimprovement}
As the proof of the regret bound for \texttt{OMS} given in \cite{oms} follows the same lines as the proof of Theorem~\ref{thm:mainFinite} given here, 
we only sketch the key step leading to the improvement of the bound.
Note that by \eqref{eq:key} and since average rewards are by assumption in $[0,1]$, for the model $\phi_{kj}$ chosen in some run $j$ of some episode~$k$ 
it holds that $\pen(\phi_{kj},t_{kj}) \leq \pen(\phi^\circ,t_{kj})-1$. 
Hence, by definition of the penalization term \eqref{eq:pen} and since $1\leq \span(\v{u}^+_{t,\phi^\circ})\leq D$, the chosen model $\phi_{kj}$ always satisfies
\begin{eqnarray*}
\lefteqn{  \big( 2 \sqrt{2S_{{kj}}}\, + \tfrac{3}{\sqrt{2}}\big) \sqrt{S_{{kj}} A \log\big(\tfrac{48 S_{{kj}} A t^3}{\delta}\big)} 
 + 2\sqrt{2\log(\tfrac{24t_{kj}^2}{\delta})}    }  \\
&\leq& 
\big( 2 D\,\sqrt{2 S^\circ}\, + \tfrac{3}{\sqrt{2}}\big) \sqrt{S^\circ A \log\big(\tfrac{48 S^\circ A t^3}{\delta}\big)}  + 2D\sqrt{2\log(\tfrac{24t_{kj}^2}{\delta})} 
+  2^{-j/2}D.
\end{eqnarray*}
Some simplifications then show that $S_{kj}$ is $\tilde{O}(D^2S^\circ)$, so that one can replace the total number of all states $S=\sum_\phi S_\phi$ in Theorem~\ref{thm:mainFinite} 
(respectively in the regret bound of \cite{oms})
by the total number of states of models $\phi$ with $S_{\phi}=\tilde{O}(D^2S^\circ)$ and consequently by $\tilde{O}(|\Phi|D^2S^\circ)$. \qed

\end{appendix}


\end{document}